\documentclass{article}

 \usepackage[preprint]{neurips_2026}

\usepackage[utf8]{inputenc} 
\usepackage[T1]{fontenc}    
\usepackage{hyperref}       
\usepackage{url}            
\usepackage{booktabs}       
\usepackage{amsfonts}       
\usepackage{nicefrac}       
\usepackage{microtype}      
\usepackage{xcolor}         

\usepackage{subcaption}
\usepackage{graphicx}
\usepackage{array}
\usepackage{amsmath}
\usepackage{amssymb}
\usepackage{mathtools}
\usepackage{amsthm}
\usepackage{algorithm}
\usepackage{algpseudocode}
\usepackage{tabularx}
\usepackage{float}
\usepackage{multirow}
\usepackage{wrapfig}
\usepackage{enumitem}

\title{TARO: Temporal Adversarial Rectification Optimization Using Diffusion Models as Purifiers}

\author{
  Daniel Wesego \\
  Department of Computer Science \\
  University of Illinois Chicago \\
  Chicago, IL 60607 \\
  \texttt{dweseg2@uic.edu} \\
  \And
  Pedram Rooshenas \\
  Department of Computer Science \\
  University of Illinois Chicago \\
  Chicago, IL 60607 \\
  \texttt{pedram@uic.edu}
}

\newcommand{\eat}[1]{}

\newtheorem{proposition}{Proposition}

\begin{document}

\maketitle

\begin{abstract}
Adversarial purification with diffusion models seeks to project adversarial examples back toward the data manifold, but balancing semantic preservation and robustness against adaptive attacks remains challenging. Recent work shows that
standard diffusion purification can fail under adaptive evaluation, while test-time score-based optimization is more resilient. Existing optimization defenses, however, typically rely on a single diffusion noise regime or treat timesteps uniformly, overlooking the distinct roles of coarse and fine denoising scales. We propose Temporal Adversarial Rectification Optimization (TARO), an inference-time purification method that builds a temporally guided score prior from multiple denoising views along the diffusion trajectory. TARO forms a coarse-to-fine residual target: high-noise experts provide globally smoothed structure with reduced adversarial sensitivity, while low-noise experts restore image-specific, class-relevant details. A guidance strength controls this temporal correction, allowing TARO to balance robust global rectification with semantic preservation. Empirically, TARO improves robust accuracy across datasets and adaptive threat models in a zero-shot setting, while remaining compatible with complementary adversarial-likelihood objectives for further robustness gains.
\end{abstract}

\section{Introduction}
\label{sec:intro}

Deep neural networks remain highly susceptible to adversarial examples, where
small, often imperceptible perturbations induce erroneous predictions under
norm-bounded input changes~\citep{szegedy2013intriguing,goodfellow2015Harnessing}.
Adversarial training is the dominant empirical defense, but its robustness is
typically tied to the perturbation family used during training, it substantially
increases training cost, and it often induces a clean-accuracy tradeoff
~\citep{madry2018Resistant,schlarmann2024robust}. These limitations have
motivated adversarial purification, an inference-time defense that preprocesses
a potentially adversarial input before classification by mapping it back toward
the natural-image manifold.

Diffusion models provide a natural prior for adversarial purification because
their denoising dynamics encode image structure across a continuum of noise
scales. Standard diffusion purifiers inject Gaussian noise into the input to
attenuate adversarial structure and then reconstruct the image through a
pretrained reverse process~\citep{song2021SDE,ho2020denoising,Nie2022DiffPure}.
This strategy is attractive because the purifier can be used with a fixed
downstream classifier. However, the diffusion time axis introduces an intrinsic
robustness--fidelity tension. Low-noise denoising preserves local image content
but may retain adversarial artifacts, while high-noise denoising suppresses
stronger perturbations but may discard class-relevant details or drift toward an
incorrect semantic mode. Prior guided purification methods attempt to steer this
reverse process using auxiliary classifiers, distance metrics, or other guidance
objectives~\citep{wang2022guided}. Yet recent adaptive evaluations have shown
that robustness claims for stochastic purification can be substantially
overstated when attacks fail to differentiate through the purification process
or fail to account for stochasticity~\citep{LeeK23_rob_eval,kassis2025diffbreak}.

A recent alternative is to cast purification as test-time optimization. Rather
than relying only on a fixed reverse diffusion trajectory, these methods
directly update the input by optimizing an objective induced by a diffusion
prior~\citep{zhang2023scoreopt,chen_robclsDiff}. Adversary-aware optimization
further strengthens this paradigm by augmenting the clean-image prior with an
explicit likelihood model for adversarial residuals~\citep{wesego2025adversary}.
Despite their empirical strength, existing optimization-based purifiers still
leave two structural limitations. First, they typically use a single diffusion
noise level at each update, collapsing the multi-scale information encoded by
the diffusion model into a single local estimate. Second, adversary-aware
variants require an auxiliary perturbation model trained from generated
adversarial examples, making the defense dependent on the attack distribution
used during training.

We propose \emph{Temporal Adversarial Rectification Optimization} (TARO), a
test-time purification framework that treats diffusion timesteps as
scale-dependent denoising experts. Different noise regimes provide
complementary information: higher-noise denoisers recover globally smoothed
structure with reduced sensitivity to fragile adversarial components, while
lower-noise denoisers preserve image-specific, class-relevant details.
Motivated by a temporal product-of-experts view of diffusion scores, TARO
constructs an affine coarse-to-fine target from multiple denoising views,
allowing the purifier to combine robust global structure with fine semantic
detail instead of committing to a single point on the robustness--fidelity
spectrum.

This temporal construction gives TARO a zero-shot rectification signal from the
pretrained clean-image diffusion model itself. Empirically, this zero-shot
setting is particularly effective under \(\ell_\infty\) adaptive attacks, where
TARO improves over prior diffusion purification and test-time optimization
baselines while preserving high clean accuracy. At the same time, our results
show that \(\ell_2\) perturbations are more challenging for raw zero-shot
temporal rectification: smoother or more spatially distributed distortions can
benefit from additional residual information or stronger stabilization. TARO is
therefore complementary to adversary-aware likelihood models such as AAOpt~\citep{wesego2025adversary}:
when an adversarial residual model is available, replacing the single-scale
diffusion prior with TARO's temporal prior further improves robustness.

Reliable evaluation is equally important for this class of defenses.
Optimization-based purifiers contain inner optimization loops, and attacks that
remove these loops, omit higher-order gradient tracking, or fail to account for
stochasticity may evaluate an easier surrogate defense rather than the actual
purifier. We therefore use an optimization-aware evaluation protocol that
attacks the full purifier--classifier composition, accounts for stochasticity
through expectation over transformation, and preserves gradient flow through
the inner loop when the optimization is evaluated. This protocol also
reveals that improperly differentiated optimization-based purification can
produce severely inflated robustness estimates, motivating our use of strong and properly applied adaptive attacks throughout the paper.

Our contributions are:
\begin{itemize}[leftmargin=1.3em, itemsep=0.5pt, topsep=2pt]
    \item We introduce TARO, a test-time purifier that aggregates multiple
    diffusion noise regimes into an affine coarse-to-fine denoising target,
    combining robust high-noise structure with fine-scale semantic preservation.

    \item We interpret TARO as an affine extension of temporal
    product-of-experts aggregation, with a correction strength controlling the
    fine-to-coarse balance, and provide risk and local Gaussian interpretations.

    \item We evaluate TARO under PGD-EOT, BPDA-EOT, AutoAttack-EOT, and
    DiffBreak. TARO improves zero-shot \(\ell_\infty\) robustness and further
    improves robustness with adversary-aware likelihoods or consistency
    regularization.

    \item We provide a optimization-aware evaluation study showing that reliable
    adaptive evaluation of optimization-based purifiers requires preserving
    higher-order gradients.
\end{itemize}


\section{Related Work}
\label{sec:related_works}

Adversarial purification is an inference-time defense motivated by the cost and
threat-model specificity of adversarial training. It preprocesses a potentially
corrupted input before classification by mapping it back toward the clean data
manifold. Diffusion purification methods, such as DiffPure~\citep{Nie2022DiffPure},
inject noise and solve a reverse denoising process to remove adversarial
structure. While effective against some attacks, standard reverse-sampling
purification faces a robustness--fidelity tradeoff: low-noise denoising may
preserve adversarial artifacts, whereas high-noise denoising can distort
class-relevant semantics.

Test-time optimization methods address this issue by directly updating the
input using a pretrained diffusion prior. Score-Opt~\citep{zhang2023scoreopt}
and LM~\citep{chen_robclsDiff} optimize score-induced objectives that move adversarially perturbed input 
\(\mathbf{x}_{\mathrm{adv}}\) toward high-density regions of the clean data
distribution. AAOpt~\citep{wesego2025adversary} strengthens this MAP view by
adding a learned adversarial likelihood
\(\log p_{\phi}(\mathbf{x}_{\mathrm{adv}}\mid\mathbf{x})\), allowing the
purifier to model structured residual artifacts. However, these methods
typically use a single score or noise scale at each update, and adversary-aware
variants require a dedicated perturbation model trained on generated attacks.
TARO instead uses multiple diffusion noise regimes at test time and can be used
either as a zero-shot temporal prior or combined with adversary-aware
likelihoods.

Robust evaluation of purification defenses is challenging because stochastic
and optimization-based purifiers can lead to misleading gradients. Prior work
has shown that diffusion purification robustness can be overestimated when
attacks are weak, non-adaptive, or fail to account for stochasticity and
differentiability through the defense~\citep{LeeK23_rob_eval,diffattacks,
kassis2025diffbreak}. DiffBreak~\citep{kassis2025diffbreak} further showed that
low-frequency, perceptually guided perturbations can collapse standard
diffusion purification. We follow adaptive evaluation best practices, including
EOT, attacks on the full purifier--classifier composition, and graph-consistent
evaluation for optimization-based purifiers whose inner loops must remain
exposed to the attacker.

Diffusion guidance provides a related algebraic perspective. Classifier
guidance modifies the score using an auxiliary classifier~\citep{dhariwal2021diffusion},
while classifier-free guidance combines conditional and unconditional denoising
estimates~\citep{ho2022classifier}. Self-Guidance~\citep{epstein2023diffusion}
and Auto-Guidance~\citep{karras2024guiding} derive guidance from internal
representations or contrasts between denoisers. These methods primarily target
controllability or sample quality in generation. In contrast, TARO uses
multi-scale denoising estimates from different diffusion noise regimes as
temporal experts for adversarial posterior rectification.

\section{Method}
\label{sec:method}

Let $\mathbf{x}_{\mathrm{adv}}\in\mathbb{R}^{d}$ be an adversarially perturbed
input and let $\mathbf{x}$ denote the unknown clean image we aim to recover.
We pose purification as explicit posterior inference:
\begin{align}\label{eq:map}
 \arg\max_{\mathbf{x}} \log p(\mathbf{x}\mid \mathbf{x}_{\mathrm{adv}}) = \arg\min_{\mathbf{x}}
\underbrace{\mathcal{L}_{\mathrm{fid}}(\mathbf{x},\mathbf{x}_{\mathrm{adv}})}
_{\text{fidelity / corruption model}}
+
\lambda
\underbrace{\mathcal{R}(\mathbf{x})}
_{\text{image prior}} .
\end{align}
In the simplest setting, we use a Gaussian corruption model,
\(
\mathcal{L}_{\mathrm{fid}}(\mathbf{x},\mathbf{x}_{\mathrm{adv}})
=
\frac{1}{2\sigma_{\mathrm{fid}}^2}
\|\mathbf{x}-\mathbf{x}_{\mathrm{adv}}\|_2^2 .
\)
More generally, $\mathcal{L}_{\mathrm{fid}}$ can be any differentiable negative
log-likelihood, including data-driven corruption models that learn
$p_\phi(\mathbf{x}_{\mathrm{adv}}\mid \mathbf{x})$~\citep{wesego2025adversary}.
In that case, we set
\(\mathcal{L}_{\mathrm{fid}}(\mathbf{x},\mathbf{x}_{\mathrm{adv}})
=
-\log p_\phi(\mathbf{x}_{\mathrm{adv}}\mid \mathbf{x}) .
\)

The MAP objective in Eq.~\eqref{eq:map} follows the standard inverse-problem
view of purification. RED replaces an explicit prior with a denoiser-induced
regularizer
\(\mathcal R_{\mathrm{RED}}(\mathbf{x})=\frac12\|\mathbf{x}-D(\mathbf{x})\|_2^2\).
For a diffusion denoiser, Tweedie's formula gives
\(D_\theta(\mathbf{y},\sigma_t)=\mathbf{y}+\sigma_t^2\nabla_{\mathbf{y}}\log
p_{\sigma_t}(\mathbf{y})\), linking denoising to score ascent. We use this
connection to build a temporal product-of-experts prior over multiple noise
levels.


\subsection{Temporal PoE and Affine Temporal Aggregation}

Let $\{p_t(\mathbf{x})\}_{t\in\mathcal T}$ denote the diffusion-induced
smoothed data distributions indexed by timesteps
$\mathcal T=[t_{\min},t_{\max}]$. We define a continuous-time temporal
product-of-experts prior:
\begin{equation}
\log p_{\mathrm{TARO}}(\mathbf{x})
=
\int_{\mathcal T} w(t)\log p_t(\mathbf{x})\,dt
-
\log Z .
\end{equation}
The function $w(t)$ controls how much each noise regime contributes to the
effective prior. For example, placing more mass on small $t$ emphasizes
fine-scale fidelity, while placing more mass on large $t$ emphasizes stronger
global correction and robustness.

By Tweedie's formula, the score of the continuous-time PoE prior can be written as
\(
\nabla_{\mathbf{x}}\log p_{\mathrm{TARO}}(\mathbf{x})
=
\int_{\mathcal T}
w(t)
\frac{D_\theta(\mathbf{x},\sigma_t)-\mathbf{x}}{\sigma_t^2}
\,dt .
\)
Defining the variance-corrected weight
$\tilde w(t)=w(t)/\sigma_t^2$, this becomes
\(
\nabla_{\mathbf{x}}\log p_{\mathrm{TARO}}(\mathbf{x})
=
\int_{\mathcal T}
\tilde w(t)
\left(D_\theta(\mathbf{x},\sigma_t)-\mathbf{x}\right)
\,dt .
\)
We approximate the continuous temporal
prior using a sparse $K$-expert quadrature:
\(\tilde w(t)\,dt
\approx
\sum_{k=1}^K \hat\gamma_k \delta(t-t_k).
\)
This gives
\(\nabla_{\mathbf{x}}\log p_{\mathrm{TARO}}(\mathbf{x})
\approx
\sum_{k=1}^K
\hat\gamma_k
\left(
D_\theta(\mathbf{x},\sigma_{t_k})-\mathbf{x}
\right).
\)
Letting \(
\Lambda=\sum_{k=1}^K\hat\gamma_k\)
and \(
\alpha_k=\frac{\hat\gamma_k}{\Lambda}
\), we obtain
\begin{equation}
\nabla_{\mathbf{x}}\log p_{\mathrm{TARO}}(\mathbf{x})
\approx
\Lambda
\left(
\sum_{k=1}^K \alpha_k D_\theta(\mathbf{x},\sigma_{t_k})
-
\mathbf{x}
\right)
=
\Lambda(\bar D_\alpha(\mathbf{x})-\mathbf{x}),
\end{equation}
where we define the temporal consensus denoiser
\(
\bar D_\alpha(\mathbf{x})
\triangleq
\sum_{k=1}^K
\alpha_kD_\theta(\mathbf{x},\sigma_{t_k}).
\) 

Standard PoE requires nonnegative expert weights, which yields a convex
interpolation of the temporal denoisers. TARO relaxes this convexity constraint
by allowing the effective aggregation weights to be negative while preserving
\(\sum_k\alpha_k=1\). When some \(\alpha_k<0\), the temporal target becomes an
affine extrapolation of the experts, enabling coarse-to-fine corrections beyond
standard convex aggregation while retaining the same denoiser-aggregation form,
\(\bar D_\alpha(\mathbf{x})\). A local Gaussian interpretation of this affine
temporal aggregation is provided in Appendix~\ref{sec:appx_affine_gaussian}.

\paragraph{Choice of temporal experts.}
The discrete experts $\{t_k\}_{k=1}^K$ are interpreted as a sparse quadrature of
the continuous temporal prior. 
The theory only requires ordered timesteps
\(t_{\min}\le t_1<\cdots<t_K\le t_{\max}\), which may be chosen by any monotone
schedule \(t_k=\psi(k;t_{\min},t_{\max},K)\), such as linear, logarithmic, or
problem-specific spacing. In our sampled-time implementation, we generate
effective timesteps by
\(t_k=\mathrm{clip}(\sqrt{t\,\kappa_k},t_{\min},t_{\max})\), with
\(\kappa_1<\cdots<\kappa_K\). We evaluate both \(K=2\) and \(K=3\): \(K=2\) uses fine and coarse experts,
while \(K=3\) adds a middle expert to stabilize the coarse-to-fine transition.

\subsection{Optimization Algorithm}
Algorithm~\ref{alg:test_time_opt_taro1} gives the score-prior-only TARO update.
At each iteration, we sample a base timestep, construct \(K\) effective
timesteps, and evaluate the denoiser at each scale. TARO anchors the target at
the coarsest expert \(u_K\) and adds lower-noise residual corrections,
\(
\mathbf{x}_{\mathrm{target}}
=
u_K+\sum_{j=1}^{K-1}\gamma^{K-j}(u_j-u_{j+1}),
\)
where \(u_j-u_{j+1}\) transfers information from a finer expert to the next
coarser one. The image is updated toward this target with the RED-style
fixed-target loss:
\(
\ell=\|\mathrm{sg}(\mathbf{x}_{\mathrm{target}})-\mathbf{x}_i\|_2^2.
\)
Here, \(\mathrm{sg}(\cdot)\) treats \(\mathbf{x}_{\mathrm{target}}\) as fixed
only during the purifier update and makes the optimization efficient similar to ~\citet{wesego2025adversary}. But to show that are our improvements are not coming from the effects this, we have added additional ablations in Section~\ref{sec:grad_flow_sanity}. Fidelity or adversary-aware
likelihood terms can be added to recover the full MAP objective.

\begin{algorithm}[!h]
  \caption{Test-time Optimization Algorithm with TARO}
  \label{alg:test_time_opt_taro1}
  \begin{algorithmic}
    \State {\bfseries Input:} Input image \(\mathbf{x}_{\mathrm{adv}}\); pretrained score denoiser \(D_\theta\); guidance strength \(\gamma\); iterations \(M\); timestep range \([t_{\min},t_{\max}]\); timestep sampling schedule; noise-level schedule \(\{\sigma_t\}_{t=1}^{T}\); ordered scale multipliers \([\kappa_1,\dots,\kappa_K]\)
    \State \(\mathbf{x}_0 \leftarrow \mathbf{x}_{\mathrm{adv}}\)
    \For{\(i=0\) {\bfseries to} \(M-1\)}
        \State Sample base timestep \(t\) from the predefined timestep schedule
        \For{\(j=1\) {\bfseries to} \(K\)}
            \State 
            \(
            t_j \leftarrow \mathrm{clip}\!\left(\sqrt{t\,\kappa_j},t_{\min},t_{\max}\right)
            \)\Comment{Set effective timestep}
            \State Sample \(\boldsymbol{\epsilon}_j \sim \mathcal N(\mathbf 0,\mathbf I)\)
            \State 
            \(
            \mathbf{u}_j \leftarrow
            D_\theta\!\left(
            \mathbf{x}_i+\sigma_{t_j}\boldsymbol{\epsilon}_j,\sigma_{t_j}
            \right)
            \) \Comment{Compute temporal expert} 
        \EndFor
        \State 
        \(
        \mathbf{x}_{\mathrm{target}}
        \leftarrow
        \mathbf{u}_K+
        \sum_{j=1}^{K-1}\gamma^{K-j}(\mathbf{u}_j-\mathbf{u}_{j+1})
        \) \Comment{Form TARO target}
        \State 
        \(
        \ell \leftarrow
        \|\mathrm{sg}(\mathbf{x}_{\mathrm{target}})-\mathbf{x}_i\|_2^2
        \) \Comment{Calculate score-prior loss}
        \State 
        \(
        \mathbf{x}_{i+1}
        \leftarrow
        \mathrm{OptimizerUpdate}(\mathbf{x}_i,\ell)
        \) \Comment{Update image 
    \EndFor}
    \State {\bfseries Return:} Final image \(\mathbf{x}_M\)
  \end{algorithmic}
\end{algorithm}

\subsection{Guidance Interpretation of Affine Temporal Extrapolation}
\label{subsec:guidance}

The affine temporal aggregation used by TARO admits a guidance-style interpretation ~\citep{ho2022classifier,karras2024guiding}.
Let \(t_f<t_c\), where \(t_f\) denotes a lower-noise fine scale and
\(t_c\) denotes a higher-noise coarse scale. We write
\(
u_f=D_\theta(x_{t_f},\sigma_{t_f})\) and \(
u_c=D_\theta(x_{t_c},\sigma_{t_c}).
\)
The coarse expert \(u_c\) provides a stronger global projection and tends to
suppress adversarial artifacts, but may introduce semantic drift or remove fine
details. The fine expert \(u_f\) is more input-grounded and better preserves
details, but may retain part of the adversarial structure. TARO uses the
temporal discrepancy \(u_f-u_c\) as a correction direction.

For \(K=2\), the TARO target is
\(
u_\gamma
=
u_c+\gamma(u_f-u_c)
=
\gamma u_f+(1-\gamma)u_c .
\)
When \(0\le\gamma\le1\), this is a convex interpolation between fine and coarse
experts. When \(\gamma>1\), it becomes a guidance-style extrapolation that
amplifies the fine-minus-coarse correction. This has the same algebraic form as
guidance methods such as CFG and Auto-Guidance, which also scale a discrepancy
between two predictions. However, the source and purpose of the discrepancy are
different: in TARO, both predictions are produced by the same diffusion denoiser
at different temporal noise regimes, and the extrapolated target is used for
test-time MAP adversarial purification rather than conditional sample
generation. Thus, TARO should be viewed as a temporal affine extrapolation with a
guidance-style form: the resemblance to guidance is algebraic, while the
correction direction is defined by fine-to-coarse denoising disagreement for
adversarial rectification.



\subsection{Risk Interpretation of Temporal Consensus}
\label{subsec:risk_interpretation}

The PoE derivation motivates aggregating temporal denoising experts. Our
implemented TARO target uses a residual coarse-to-fine correction controlled by
\(\gamma\). We therefore analyze the risk of the guided fine-to-coarse target
directly, which makes the role of \(\gamma\) explicit.

\begin{proposition}[Risk of the guided fine-to-coarse target]
\label{prop:gamma_guided_risk}
Let \(x_{t_f}=x+\sigma_{t_f}\epsilon_f\) and
\(x_{t_c}=x+\sigma_{t_c}\epsilon_c\). We use \(u_f=D_\theta(x_{t_f},\sigma_{t_f})\) and
\(u_c=D_\theta(x_{t_c},\sigma_{t_c})\) to denote the fine and coarse temporal experts,
where \(t_f<t_c\). Suppose
\(
u_f=x^\star+b_f(x^\star)+\xi_f\) and 
\(u_c=x^\star+b_c(x^\star)+\xi_c
\)
with
\(\mathbb E[\xi_f\mid x^\star]=\mathbb E[\xi_c\mid x^\star]=0\).
Define the TARO guided target
\(
u_\gamma=u_c+\gamma(u_f-u_c).
\) Then
\[
u_\gamma
=
x^\star+
\underbrace{b_c+\gamma(b_f-b_c)}_{b_\gamma}
+
\underbrace{\xi_c+\gamma(\xi_f-\xi_c)}_{\xi_\gamma}, \;\;\text{and} \]
\[
\mathbb E[\|u_\gamma-x^\star\|_2^2\mid x^\star]
=
\|b_c+\gamma(b_f-b_c)\|_2^2
+
\mathbb E[\|\xi_c+\gamma(\xi_f-\xi_c)\|_2^2\mid x^\star].
\]
\end{proposition}

Proof of Proposition~\ref{prop:gamma_guided_risk} is provided in
Appendix~\ref{appx:proof}.

Proposition~\ref{prop:gamma_guided_risk} shows that the correction strength
\(\gamma\) changes both the effective bias \(b_\gamma\) and the residual
variation \(\xi_\gamma\) of the temporal target. Small \(\gamma\) keeps
\(u_\gamma\) close to the coarse expert \(u_c\), which can inherit
oversmoothing bias and remove class-relevant details. Larger \(\gamma\)
increases the influence of the fine expert \(u_f\) through the correction
\(u_f-u_c\), restoring image-specific structure but potentially amplifying
cross-timestep residual variation. This motivates selecting \(\gamma\) on
validation data and using expert consistency regularization when stronger
fine-expert correction becomes unstable.

This interpretation also helps explain the attack-dependent behavior observed in
our ablations. For \(\ell_\infty\) attacks, moderate fine-to-coarse correction
is effective, but overly large \(\gamma\) can over-amplify temporal residuals
and reduce robustness. For \(\ell_2\) attacks, smoother and more spatially
distributed distortions are less fully corrected by the coarse expert alone, so
stronger fine-expert correction is beneficial. In this regime, consistency
regularization becomes especially useful because it stabilizes the amplified
cross-timestep discrepancy.

\paragraph{Expert consistency regularization.}
Because larger \(\gamma\) amplifies the temporal correction \(u_f-u_c\), it can
also amplify unstable cross-timestep disagreement. To stabilize the guided
target, we augment the score-prior objective with an expert-consistency
regularizer:
\(
\mathcal L_{\mathrm{cons}}
=
\sum_{i<j}\|u_i-u_j\|_2^2,
\)
where \(u_i\) and \(u_j\) are denoised outputs from different temporal experts.
The full TARO prior loss becomes
\(
\mathcal L_{\mathrm{prior}}
=
\|\mathrm{sg}(x_{\mathrm{target}})-x\|_2^2
+
\lambda_{\mathrm{cons}}\mathcal L_{\mathrm{cons}}.
\)
The first term moves the current image toward the guided temporal target, while
\(\mathcal L_{\mathrm{cons}}\) encourages temporal experts to agree on the
purified estimate. From the risk view above, this regularizer is intended to
control unstable residual variation while preserving the bias-correction
benefit of stronger fine-scale guidance. This is particularly useful for
\(\ell_2\) attacks, where the fine expert remains important but the raw
cross-timestep residual benefits from stabilization. For \(\ell_\infty\)
attacks, where temporal discrepancy may already provide a useful correction
signal, excessive consistency can dampen that signal and provide less or no
benefit. See Appendix~\ref{sec:gamma_ablation} for the effect of consistency loss on purification against $\ell_2$ attacks.

\subsection{Using Data-Driven Adversarial Models}
\label{subsec:data_driven_adversarial_model}

The risk interpretation above suggests that purification can benefit from both
temporal correction and attack-specific residual modeling. Adversary-aware
methods such as AAOpt~\citep{wesego2025adversary} address the latter by
explicitly modeling the residual between the adversarial input and the purified
estimate. In particular, AAOpt augments the diffusion prior with a learned
perturbation likelihood by minimizing an objective of the form
\(
\mathcal L_{\mathrm{AAOpt}}(\mathbf{x})
=
-\log p_\theta(\mathbf{x})
-
\lambda
\log p_{\mathrm{pert}}(\mathbf{x}-\mathbf{x}_{\mathrm{adv}}).
\)
Here, \(p_{\mathrm{pert}}(\mathbf{x} -\mathbf{x}_{\mathrm{adv}})\) models the
likelihood of the residual between the current
purified estimate and the adversarial input.

TARO is complementary: it uses the zero-shot discrepancy
\(\Delta u=u_f-u_c\) between fine and coarse temporal experts as an internal
proxy for scale-sensitive adversarial structure. Thus, AAOpt provides an
external learned residual model, while TARO provides an internal temporal
rectification signal from the pre-trained diffusion model itself. The two can
be combined by replacing the single-scale diffusion prior in AAOpt with the
TARO temporal consensus prior (see Appendix~\ref{sec:appx_taro_aa}).

\section{Experiments} 
\label{sec:exp}

\paragraph{Setup.} We evaluate TARO on CIFAR-10, CIFAR-100, CIFAR-10-C, and Imagenette. CIFAR-10-C, Imagenette, and extended hyperparameter details are discussed in \hyperref[sec:appx]{Appendix}. We report clean accuracy on benign inputs and robust accuracy under adaptive attacks. Unless otherwise stated, TARO results are averaged over independent trials.

We compare against adversarial training (AT) baselines, including \cite{pang2022robustness}, \cite{gowal2020uncovering}, \cite{gowal2021improving}, and \cite{wang2023better}; standard adversarial purification (AP) methods, including \cite{Nie2022DiffPure}, \cite{yoon2021adp}, \cite{rob_diff_guang}, and \cite{li2025adbm}; and test-time optimization baselines, including LM~\citep{chen_robclsDiff}, Score-Opt~\citep{zhang2023scoreopt}, and AAOpt~\citep{wesego2025adversary}. For direct comparisons, we reimplement AAOpt using the same classifier, diffusion checkpoint, and attack settings as TARO.

We use pretrained EDM denoisers that predict the clean estimate
\(\hat x_0=D_\theta(x_t;\sigma)\), together with clean-trained WRN-28-10 and
WRN-70-16 classifiers. Thus, robustness improvements for AP methods come from
purification rather than adversarially trained classifiers. TARO-\(K\) denotes
aggregation over \(K\) diffusion noise scales, and TARO-\(K\)-AA denotes the
hybrid variant by updating AAOpt with the multi-scale temporal property of TARO.

\paragraph{Evaluation protocol.}
Because TARO is an optimization-based purifier, attacks must target the full
purifier--classifier composition and account for stochasticity with EOT. We use
PGD-EOT, BPDA-EOT, AutoAttack-EOT, and DiffBreak-style adaptive attacks. We also
provide an optimization-aware evaluation study in Appendix~\ref{sec:appx_proper_eval},~\ref{sec:appx_clipure_eval},
showing that ignoring higher-order gradients in optimization based purification can lead to severely inflated robustness estimates.

\subsection{Guidance Strength and Consistency Ablation}
\label{sec:gamma_ablation}

We first study the temporal correction strength \(\gamma\) in TARO-2.
Small \(\gamma\) keeps the target closer to the high-noise coarse expert, while
larger \(\gamma\) increases the contribution of the low-noise fine expert. We
also evaluate an expert-consistency regularizer that penalizes disagreement
among denoised timestep outputs, stabilizing the amplified temporal correction.
Table~\ref{tab:gamma_consistency_ablation_cifar10_pgd_eot} reports results on
CIFAR-10 with WRN-28-10 under PGD-EOT-20 for both \(\ell_\infty\)
\((\epsilon=8/255)\) and \(\ell_2\) \((\epsilon=0.5)\).

Table~\ref{tab:gamma_consistency_ablation_cifar10_pgd_eot} shows that the fine
expert plays an important role in both clean and robust accuracy. When \(\gamma<1\), clean accuracy and robust accuracy both drop, which is
consistent with the target remaining too close to the coarse expert and losing
class-relevant detail. Without consistency, \(\ell_\infty\) robustness peaks at
\(\gamma=1.5\), while larger values reduce robustness, suggesting that excessive correction can
over-amplify unstable temporal residuals. In
contrast, \(\ell_2\) robustness improves with stronger correction up to
\(\gamma=2.0\).

Consistency regularization is especially important for \(\ell_2\): at
\(\gamma=1.5\), it improves robust accuracy from \(69.33\%\) to \(89.05\%\).
This indicates that the fine-to-coarse correction is useful for \(\ell_2\)
attacks, but the raw cross-timestep discrepancy benefits from stabilization.
For \(\ell_\infty\), consistency has a smaller effect near the best
\(\gamma\), mainly stabilizing the large-\(\gamma\) regime.

\begin{table}[!h]
\centering
\caption{Guidance-strength and consistency ablation for TARO-2 on CIFAR-10
using WRN-28-10 under PGD-EOT-20. We report clean accuracy and robust accuracy
under \(\ell_\infty\) and \(\ell_2\) attacks.}
\label{tab:gamma_consistency_ablation_cifar10_pgd_eot}
\small
\setlength{\tabcolsep}{4pt}
\begin{tabular}{@{}c c c c c@{}}
\toprule
\(\gamma\) & Consistency & Clean Acc. &
\(\ell_\infty\) Robust Acc. &
\(\ell_2\) Robust Acc. \\
\midrule
0.5  & No  & 82.28\(\scriptstyle\pm0.7\) & 73.56\(\scriptstyle\pm2.2\) & 57.48\(\scriptstyle\pm0.6\) \\
0.5  & Yes & 81.83\(\scriptstyle\pm1.3\) & 73.82\(\scriptstyle\pm1.4\) & 75.51\(\scriptstyle\pm1.5\) \\
0.75 & No  & 85.73\(\scriptstyle\pm0.7\) & 79.55\(\scriptstyle\pm0.5\) & 61.00\(\scriptstyle\pm1.0\) \\
0.75 & Yes & 84.50\(\scriptstyle\pm0.9\) & 78.51\(\scriptstyle\pm0.8\) & 79.42\(\scriptstyle\pm0.6\) \\
1.0  & No  & 91.46\(\scriptstyle\pm0.1\) & 85.93\(\scriptstyle\pm0.4\) & 63.40\(\scriptstyle\pm1.5\) \\
1.0  & Yes & 89.71\(\scriptstyle\pm0.5\) & 84.43\(\scriptstyle\pm0.3\) & 85.28\(\scriptstyle\pm0.4\) \\
1.5  & No  & \textbf{93.09}\(\scriptstyle\pm0.4\) & {87.04}\(\scriptstyle\pm1.8\) & 69.33\(\scriptstyle\pm1.2\) \\
1.5  & Yes & 92.77\(\scriptstyle\pm0.7\) & \textbf{87.82}\(\scriptstyle\pm1.4\) & \textbf{89.05}\(\scriptstyle\pm0.5\) \\
2.0  & No  & 92.96\(\scriptstyle\pm0.6\) & 79.68\(\scriptstyle\pm2.3\) & 72.91\(\scriptstyle\pm2.3\) \\
2.0  & Yes & 92.83\(\scriptstyle\pm0.8\) & 80.90\(\scriptstyle\pm1.9\) & 86.39\(\scriptstyle\pm1.5\) \\
3.0  & No  & 91.46\(\scriptstyle\pm0.9\) & 55.79\(\scriptstyle\pm1.1\) & 66.27\(\scriptstyle\pm0.8\) \\
3.0  & Yes & 91.40\(\scriptstyle\pm1.3\) & 58.13\(\scriptstyle\pm0.7\) & 74.86\(\scriptstyle\pm0.7\) \\
\bottomrule
\end{tabular}
\end{table}

\subsection{Adaptive PGD-EOT Attacks}
\label{sec:pgd_attack}
We evaluate end-to-end robustness against adaptive Projected Gradient
Descent with Expectation over Transformation (PGD-EOT)~\citep{athalye2018obfuscated}.
The attack targets the full purifier--classifier composition and estimates
gradients through the stochastic purification process. We consider both
\(\ell_\infty\) and \(\ell_2\) perturbations with budgets
\(\epsilon=8/255\) and \(\epsilon=0.5\), respectively. Following
\citet{wesego2025adversary}, we use 20 attack iterations and a one-step
gradient approximation through the recurrent purification trajectory. Unlike
the original AAOpt evaluation, which uses 20 EOT samples for the attack and a
larger number of EOT samples for the defense, we use 20 EOT samples for both
attack and defense across TARO and all directly reimplemented baselines. This
gives a matched and computationally comparable evaluation setting. 

Table~\ref{tab:pgd_cifar10_linf} reports CIFAR-10 results under
\(\ell_\infty\) and \(\ell_2\) PGD-EOT using a WRN-28-10 classifier. Under
\(\ell_\infty\), TARO achieves the strongest robustness among the evaluated
purification methods while maintaining clean accuracy close to the undefended
classifier. Appendix~\ref{sec:appx_wrn_70_16} reports the corresponding
WRN-70-16 results, where TARO remains strong: TARO-2 achieves the best robust
accuracy at \(88.56\%\), and TARO-3 obtains a comparable robust accuracy of
\(88.37\%\).

Under \(\ell_2\), raw zero-shot TARO is weaker, but both TARO-AA and
consistency-regularized TARO recover strong robustness. In particular,
TARO-2 + Cons. achieves the best \(\ell_2\) robustness in
Table~\ref{tab:pgd_cifar10_linf}, suggesting that stabilizing the temporal
correction can be competitive with adding an adversary-aware likelihood in this
setting.


\begin{table}[!h]
\centering
\caption{Standard and robust accuracy on CIFAR-10 using WRN-28-10 under PGD-EOT-20. Baselines include adversarial training (AT), adversarial purification, and test-time optimization purification.}
\label{tab:pgd_cifar10_linf}
\small
\setlength{\tabcolsep}{3pt}

\begin{minipage}{0.49\textwidth}
\centering
\textbf{$\ell_{\infty}$ ($\epsilon = 8/255$)}

\begin{tabular}{@{}l l c c@{}}
\toprule
Type & Method & Clean & Robust \\
\midrule
& Default & 94.78 & 0.0 \\
\midrule
\multirow{3}{*}{AT} 
& \citet{pang2022robustness} & 88.62 & 64.95 \\
& \citet{gowal2020uncovering} & 88.54 & 65.10 \\
& \citet{gowal2021improving} & 87.51 & 66.01 \\
\midrule
\multirow{5}{*}{AP} 
& \citet{yoon2021adp} & 85.66 & 37.27 \\
& \citet{Nie2022DiffPure} & 91.41 & 51.25 \\
& Score-Opt-O & 91.21 & 64.96 \\
& Score-Opt-N & 94.43 & 65.62 \\
& \citet{rob_diff_guang} & 90.42 & 64.06 \\
\midrule
& AAOpt & 91.53$\scriptstyle\pm0.9$ & 85.12$\scriptstyle\pm0.7$ \\
\midrule
& TARO-2 & 93.41$\scriptstyle\pm1.1$ & 88.27$\scriptstyle\pm1.6$ \\
& TARO-3 & 93.45$\scriptstyle\pm0.9$ & \textbf{88.57}$\scriptstyle\pm0.9$ \\
& TARO-2-AA & 92.61$\scriptstyle\pm0.7$ & 86.91$\scriptstyle\pm0.9$ \\
& TARO-3-AA & 93.45$\scriptstyle\pm0.9$ & 87.46$\scriptstyle\pm0.7$ \\
\bottomrule
\end{tabular}
\end{minipage}
\hfill
\begin{minipage}{0.49\textwidth}
\centering
\textbf{$\ell_{2}$ ($\epsilon = 0.5$)}
  \begin{tabular}{@{}l l c c@{}}
  \toprule
  Type & Method & Clean & Robust \\
  \midrule
  & Default & 94.78 & 0.0 \\
  \midrule
  \multirow{3}{*}{AT} 
  & \citet{rebuffi2021fixing} & 91.79 & 85.05 \\
  & \citet{augustin2020adversarial} & 93.96 & 86.14 \\
  \midrule
  \multirow{5}{*}{AP} 
  & \citet{Nie2022DiffPure} & 91.41 & 82.11 \\
  & Score-Opt-O & 91.21 & 79.09 \\
  & Score-Opt-N & 94.43 & 84.86 \\
  & \citet{rob_diff_guang} & 90.42 & 85.55 \\
  \midrule
  & AAOpt & 91.48$\scriptstyle\pm0.7$ & 86.40$\scriptstyle\pm0.4$ \\
  \midrule
  & TARO-2 & 93.43$\scriptstyle\pm0.6$ & 69.21$\scriptstyle\pm0.7$ \\
  & TARO-3 & 93.55$\scriptstyle\pm1.1$ & 68.86$\scriptstyle\pm1.1$ \\
  & TARO-2-AA & 92.61$\scriptstyle\pm0.7$ & 88.19$\scriptstyle\pm1.2$ \\
  & TARO-3-AA & 93.47$\scriptstyle\pm0.9$ & {88.39}$\scriptstyle\pm1.0$ \\
  & TARO-2 + Cons. & 93.43$\scriptstyle\pm0.8$ & \textbf{88.99}$\scriptstyle\pm1.2$\\
  & TARO-3 + Cons. & 91.27$\scriptstyle\pm0.3$ & 86.97$\scriptstyle\pm0.3$\\
  \bottomrule
  \end{tabular}
\end{minipage}
\end{table}

\subsection{BPDA-EOT Attacks}
\label{sec:bpda_attack}
We next evaluate robustness against Backward Pass Differentiable Approximation
with EOT (BPDA-EOT)~\citep{athalye2018obfuscated}. BPDA-EOT tests defenses with
non-differentiable, stochastic, or approximate components by replacing the
backward pass with a differentiable surrogate while averaging gradients over
stochastic transformations. We use 50 BPDA iterations and 15 EOT samples under
an \(\ell_\infty\) perturbation budget of \(\epsilon=8/255\). This setting
tests whether TARO remains robust when the attack accounts for randomized
purification and approximate gradients. Tables~\ref{tab:bpda_eot_cif10}
and~\ref{tab:bpda_eot_cif100} summarize the BPDA-EOT results on CIFAR-10 and
CIFAR-100. On CIFAR-10, TARO-3-AA achieves the strongest robust accuracy while
maintaining high clean accuracy. On CIFAR-100, where classification is more
challenging and the data manifold is more complex, TARO-3-AA again improves
over AAOpt and Score-Opt baselines. These results support the claim that
temporal multi-scale purification provides a stronger optimization signal than
single-scale score guidance.

\begin{table}[!h]
  \centering
  \vspace{-0.15in}
  \small
  \setlength{\tabcolsep}{4pt}

  \parbox[t]{0.48\textwidth}{
  \centering
  \captionof{table}{BPDA+EOT attack on CIFAR-10 under $\ell_\infty$ perturbations ($\epsilon=8/255$), WRN-28-10.}
  \label{tab:bpda_eot_cif10}
  \begin{tabular}{@{}l c c@{}}
    \toprule
    Method & Clean & Robust \\
    \midrule
    Default & 94.78 & 0.0 \\ 
    \midrule
    \citet{wang2022guided} & 93.50 & 79.83 \\
    \citet{yoon2021adp} & 86.14 & 70.01 \\
    \citet{Nie2022DiffPure} & 89.02 & 81.40 \\
    Score-Opt-O & 90.23 & 81.36 \\
    Score-Opt-N & 93.94 & 90.07 \\
    ADBM & 91.93 & 70.51 \\
    \midrule
    AAOpt & 92.10$\scriptstyle\pm0.7$ & 91.36$\scriptstyle\pm0.8$ \\
    \midrule
    TARO-2 & 90.62$\scriptstyle\pm0.6$ & 88.90$\scriptstyle\pm0.6$ \\
    TARO-3 & 89.40$\scriptstyle\pm0.4$ & 86.87$\scriptstyle\pm0.6$ \\
    TARO-2-AA & 92.88$\scriptstyle\pm0.8$ & 92.49$\scriptstyle\pm0.8$ \\
    TARO-3-AA & 93.59$\scriptstyle\pm0.7$ & \textbf{93.04}$\scriptstyle\pm0.8$ \\
    \bottomrule
  \end{tabular}
  }
  \hfill
  \parbox[t]{0.48\textwidth}{
  \centering
  \captionof{table}{BPDA+EOT attack on CIFAR-100 under $\ell_\infty$ perturbations ($\epsilon=8/255$), WRN-28-10.}
  \label{tab:bpda_eot_cif100}
  \begin{tabular}{@{}l c c@{}}
    \toprule
    Method & Clean & Robust \\
    \midrule
    Default & 81.55 & 0.0 \\
    \midrule
    \citet{yoon2021adp} & 60.66 & 39.72 \\ 
    \citet{hill2021ebm} & 51.66 & 26.10 \\
    Score-Opt-O & 70.53 & 66.11 \\
    Score-Opt-N & 74.18 & 60.21 \\
    \midrule
    AAOpt & 69.68$\scriptstyle\pm1.6$ & 66.20$\scriptstyle\pm1.8$ \\
    \midrule
    TARO-2 & 73.78$\scriptstyle\pm1.5$ & 66.71$\scriptstyle\pm1.9$ \\
    TARO-3 & 72.61$\scriptstyle\pm1.8$ & 68.66$\scriptstyle\pm2.2$ \\
    TARO-2-AA & 72.76$\scriptstyle\pm1.9$ & 70.65$\scriptstyle\pm2.0$ \\
    TARO-3-AA & 74.72$\scriptstyle\pm1.5$ & \textbf{72.41}$\scriptstyle\pm1.7$ \\
    \bottomrule
  \end{tabular}
  }
\end{table}

\subsection{DiffBreak Low-Frequency Attacks}
\label{sec:diffbreak}
\vspace{-0.05in}
Recently, \citet{kassis2025diffbreak} introduced DiffBreak, an adaptive attack
designed specifically to bypass diffusion-based purification defenses. Unlike
standard \(\ell_p\)-bounded attacks, DiffBreak uses a low-frequency (LF)
perturbation strategy guided by perceptual similarity objectives such as LPIPS.
These perturbations remain closer to natural image statistics, making them
difficult for standard diffusion denoisers to remove. As shown in
Table~\ref{tab:lf_attk_cif10}, the evaluated standard diffusion purification
baselines collapse under this attack: both one-step denoising and DiffPure-EDM
achieve \(0.0\%\) robust accuracy. In contrast, AAOpt and TARO-2-AA remain
substantially more robust, with TARO-2-AA improving over AAOpt.

\begin{table}[!h]
\centering

\parbox[t]{0.48\textwidth}{
\centering
\captionof{table}{Robust accuracy on CIFAR-10 under the DiffBreak low-frequency (LF) attack with WRN-28-10 and EOT=20.}
\label{tab:lf_attk_cif10}
\small
\setlength{\tabcolsep}{5pt}
\begin{tabular}{@{}l c@{}}
\toprule
Model & Robust Accuracy \\
\midrule
1-step Denoise & 0.0$\scriptstyle\pm0.0$ \\
DiffPure-EDM & 0.0$\scriptstyle\pm0.0$ \\
AAOpt & 85.41$\scriptstyle\pm4.4$ \\
TARO-2-AA & \textbf{90.10}$\scriptstyle\pm1.5$ \\
\bottomrule
\end{tabular}
}
\hfill
\parbox[t]{0.48\textwidth}{
\centering
\captionof{table}{Robust accuracy on CIFAR-10 under transfer-DiffBreak low-frequency (LF) attack with similar setup as Table ~\ref{tab:lf_attk_cif10}}
\label{tab:transfer_lf_attk_cif10}
\small
\setlength{\tabcolsep}{5pt}
\begin{tabular}{@{}l c@{}}
\toprule
Model & Robust Accuracy \\
\midrule
DiffPure-EDM & 0.0$\scriptstyle\pm0.0$ \\
AAOpt & 93.2$\scriptstyle\pm2.9$ \\
TARO-2-AA & \textbf{94.7}$\scriptstyle\pm1.4$ \\
\bottomrule
\end{tabular}
}
\end{table}

As an additional sanity check, we also evaluate a transfer-DiffBreak setting. We
generate low-frequency adversarial examples against DiffPure-EDM using the same
diffusion model and evaluate the resulting examples on TARO with matched seed
settings. Table~\ref{tab:transfer_lf_attk_cif10} shows that TARO-2-AA remains
robust under transfer, again outperforming AAOpt. For the direct LF attack, we
use 20 EOT samples; due to the high computational overhead of sequential EOT
execution and 100 attack steps, we evaluate 64 samples averaged over 3 trials.


\subsection{Gradient-Flow Sanity Check}
\label{sec:grad_flow_sanity}

Because TARO uses a fixed-target loss with \(\mathrm{sg}(\cdot)\), we verify
that its robustness does not rely on gradient obstruction. Table ~\ref{tab:main_stopgrad_ablation}
reports an ablation that removes the stop-gradient operation from the temporal
target. TARO remains robust without this detachment: for example, TARO-3
achieves \(87.13\%\) robust accuracy without stop-gradient compared with
\(88.57\%\) with stop-gradient under CIFAR-10 \(\ell_\infty\) PGD-EOT. This
suggests that the gains come from temporal multi-scale rectification rather
than broken gradients. 
\begin{table}[!h]
\vspace{-0.1in}
\centering
\caption{Stop-gradient ablation on CIFAR-10 under PGD-EOT-20
\((\ell_\infty,\epsilon=8/255)\), WRN-28-10.}
\label{tab:main_stopgrad_ablation}
\small
\setlength{\tabcolsep}{6pt}
\begin{tabular}{@{}lcc@{}}
\toprule
Model & Stop-grad & No stop-grad \\
\midrule
TARO-2 & 88.27\(\scriptstyle\pm1.6\) & 86.51\(\scriptstyle\pm1.9\) \\
TARO-3 & \textbf{88.57}\(\scriptstyle\pm0.9\) & \textbf{87.13}\(\scriptstyle\pm1.1\) \\
TARO-2-AA & 86.91\(\scriptstyle\pm0.9\) & 85.86\(\scriptstyle\pm1.2\) \\
TARO-3-AA & 87.46\(\scriptstyle\pm0.7\) & 86.19\(\scriptstyle\pm1.1\) \\
\bottomrule
\end{tabular}
\end{table}

\subsection{Full-Reverse AutoAttack-EOT}
\label{sec:autoattack_eot}

We further evaluate TARO with AutoAttack-EOT by considering the full
reverse purification trajectory, rather than using the one-step approximation
from the main PGD-EOT evaluation. This provides a stronger adaptive sanity check
against incomplete-gradient effects. The attack targets the full
purifier--classifier pipeline and averages gradients over stochastic
purification samples. To reduce cost, we evaluate 128 samples. Table~\ref{tab:autoattack_eot_full_reverse} reports the results. TARO-2-AA achieves the strongest robustness under this stricter unrolled attack,
improving over AAOpt while preserving clean accuracy. These results indicate
that TARO's temporal prior remains effective under full adaptive evaluation
when integrated with adversary-aware optimization.

\begin{table}[!h]
\centering

\vspace{-0.1in}
\caption{AutoAttack-EOT with full reverse optimization on CIFAR-10, WRN-28-10.}
\label{tab:autoattack_eot_full_reverse}
\small
\setlength{\tabcolsep}{5pt}
\begin{tabular}{@{}l c c@{}}
\toprule
Model & Clean Acc. & Robust Acc. \\
\midrule
Default & 94.78 & 0.0 \\
Score-Opt-O & 89.84 & 74.22 \\
AAOpt & 89.06 & 85.54 \\
\midrule
TARO-2 & 90.62 & 71.09 \\
TARO-2-AA & 91.01 & \textbf{87.50} \\
\bottomrule
\end{tabular}
\end{table}

\section{Conclusion}
\label{sec:concl}

We proposed TARO, a test-time purification framework that treats diffusion
timesteps as complementary denoising experts. TARO builds an affine
coarse-to-fine target that combines robust global structure from high-noise
experts with class-relevant detail from low-noise experts. Empirically, TARO
improves robustness under strong adaptive attacks while maintaining high clean
accuracy. Zero-shot TARO is especially effective under \(\ell_\infty\) attacks,
while \(\ell_2\) robustness benefits from consistency regularization or
adversary-aware likelihoods. We also highlight the need for graph-consistent
adaptive evaluation of optimization-based purifiers.

\paragraph{Limitations.}
TARO increases inference latency due to iterative multi-scale optimization, and
its correction strength and stabilization regularization must be selected carefully for each threat model; see Appendix~\ref{sec:appx_limitations} for a
detailed discussion.



\bibliography{bibliography}
\bibliographystyle{plainnat}

\newpage
\appendix
\section{Appendix}
\label{sec:appx}
This appendix provides additional algorithmic, evaluation, and experimental details omitted from the main text due to space constraints. 
Appendix~\ref{sec:appx_affine_gaussian} provides a local Gaussian
interpretation of affine temporal extrapolation.
Appendix~\ref{appx:proof} shows proof of proposition 1, then Appendix~\ref{sec:appx_taro_aa} describes the hybrid TARO-AA variant that combines temporal multi-scale purification with the adversary-aware perturbation likelihood of AAOpt. Appendix~\ref{sec:appx_proper_eval} formalizes the evaluation protocol used for optimization-based purifiers, and Appendix~\ref{sec:appx_clipure_eval} presents a corrected-evaluation case study showing that improperly differentiated optimization loops can lead to severely inflated robustness estimates. Appendix~\ref{sec:appx_wrn_70_16} reports additional results using the WRN-70-16 classifier, and Appendix~\ref{sec:imagenette_results} reports results using the Imagenette dataset. Appendix~\ref{sec:time_scale_ablation} includes additional ablations on time scales, Appendix~\ref{sec:corr} reports the results against common corruptions using the CIFAR-10-C dataset, Appendix ~\ref{sec:compute} discusses the computation requirement and compares the wall clock time the algorithm takes, Appendix~\ref{sec:exp_details} lists the hyperparameters used in our experiments, and finally, Appendix~\ref{sec:appx_limitations} discusses the limitation in detail.

\subsection{Local Gaussian View of Affine Temporal Extrapolation}
\label{sec:appx_affine_gaussian}

Positive-weight temporal PoE motivates aggregating diffusion experts, but it
only yields interpolation. TARO uses a more general affine log-density
combination,
\[
\log p_{\mathrm{aff}}(x)
=
\sum_{k=1}^K \gamma_k\log p_{t_k}(x)-\log Z,
\qquad
\sum_{k=1}^K\gamma_k=1,
\]
where the coefficients need not be nonnegative. When all \(\gamma_k\ge0\), this
reduces to standard PoE interpolation. When some coefficients are negative, it
becomes a contrastive or ratio-of-experts prior. For \(K=2\),
\[
\log p_\gamma(x)
=
\gamma\log p_f(x)+(1-\gamma)\log p_c(x),
\]
and for \(\gamma>1\),
\[
p_\gamma(x)
\propto
\frac{p_f(x)^\gamma}{p_c(x)^{\gamma-1}}.
\]
Thus, affine temporal extrapolation emphasizes structures supported by the fine
expert relative to the coarse expert.

Under a local Gaussian approximation
\(p_{t_k}(x)\approx\mathcal N(\mu_{t_k},\Sigma_{t_k})\), the affine
log-density has local precision
\[
\Lambda_{\mathrm{aff}}
=
\sum_{k=1}^K\gamma_k\Lambda_{t_k},
\]
provided this matrix is positive definite. If the fine expert has larger local
precision than the coarse experts, and the affine weights place extrapolative
positive mass on the fine expert while subtracting coarser experts, then
\[
\Lambda_{\mathrm{aff}}
=
\Lambda_f+\sum_{k>1}a_k(\Lambda_f-\Lambda_{t_k})
\succeq
\Lambda_f .
\]
This gives local intuition for why emphasizing the fine expert can sharpen the
target and restore class-relevant structure. However, excessive sharpening may
also amplify unstable cross-scale residuals, motivating validation-based
selection of \(\gamma\) and the consistency regularizer used in our ablations.

It should be noted that unlike positive-weight PoE, affine extrapolation does not automatically
preserve positive definiteness. The local Gaussian interpretation is valid only
when \(\Lambda_{\mathrm{aff}}\succ0\). In the two-expert case with
\(\gamma>1\), a sufficient condition is \(\Lambda_f\succeq\Lambda_c\), in which
case
\[
\Lambda_{\mathrm{aff}}
=
\Lambda_f+(\gamma-1)(\Lambda_f-\Lambda_c)
\succeq
\Lambda_f
\succ0.
\]

\subsection{Proof of Proposition 1}\label{appx:proof}
\begin{proof}
By assumption, the fine and coarse temporal experts admit the decompositions
\[
u_f=x^\star+b_f(x^\star)+\xi_f,
\qquad
u_c=x^\star+b_c(x^\star)+\xi_c,
\]
with
\(\mathbb E[\xi_f\mid x^\star]=\mathbb E[\xi_c\mid x^\star]=0\).
The TARO guided target is
\[
u_\gamma
=
u_c+\gamma(u_f-u_c).
\]
Substituting the decompositions of \(u_f\) and \(u_c\), we obtain
\[
u_f-u_c
=
b_f(x^\star)-b_c(x^\star)+\xi_f-\xi_c.
\]
Therefore,
\[
u_\gamma
=
x^\star+b_c(x^\star)+\xi_c
+
\gamma\left(
b_f(x^\star)-b_c(x^\star)+\xi_f-\xi_c
\right).
\]
Rearranging terms gives
\[
u_\gamma
=
x^\star
+
\underbrace{
b_c(x^\star)+\gamma\bigl(b_f(x^\star)-b_c(x^\star)\bigr)
}_{b_\gamma(x^\star)}
+
\underbrace{
\xi_c+\gamma(\xi_f-\xi_c)
}_{\xi_\gamma}.
\]
Thus,
\[
u_\gamma-x^\star
=
b_\gamma(x^\star)+\xi_\gamma.
\]
Taking the squared norm,
\[
\|u_\gamma-x^\star\|_2^2
=
\|b_\gamma(x^\star)+\xi_\gamma\|_2^2.
\]
Expanding,
\[
\|u_\gamma-x^\star\|_2^2
=
\|b_\gamma(x^\star)\|_2^2
+
2b_\gamma(x^\star)^\top \xi_\gamma
+
\|\xi_\gamma\|_2^2.
\]
Now take conditional expectation given \(x^\star\). Since
\[
\xi_\gamma
=
\xi_c+\gamma(\xi_f-\xi_c),
\]
and both \(\xi_f\) and \(\xi_c\) have zero conditional mean, we have
\[
\mathbb E[\xi_\gamma\mid x^\star]
=
\mathbb E[\xi_c\mid x^\star]
+
\gamma
\left(
\mathbb E[\xi_f\mid x^\star]
-
\mathbb E[\xi_c\mid x^\star]
\right)
=
0.
\]
Therefore, the cross term vanishes:
\[
\mathbb E[
2b_\gamma(x^\star)^\top \xi_\gamma
\mid x^\star]
=
2b_\gamma(x^\star)^\top
\mathbb E[\xi_\gamma\mid x^\star]
=
0.
\]
Hence,
\[
\mathbb E[
\|u_\gamma-x^\star\|_2^2
\mid x^\star]
=
\|b_\gamma(x^\star)\|_2^2
+
\mathbb E[
\|\xi_\gamma\|_2^2
\mid x^\star].
\]
Substituting the definitions of \(b_\gamma\) and \(\xi_\gamma\), we obtain
\[
\mathbb E[
\|u_\gamma-x^\star\|_2^2
\mid x^\star]
=
\|b_c+\gamma(b_f-b_c)\|_2^2
+
\mathbb E[
\|\xi_c+\gamma(\xi_f-\xi_c)\|_2^2
\mid x^\star],
\]
where we suppress the explicit dependence of \(b_f\) and \(b_c\) on
\(x^\star\) for readability.

This proves the proposition.
\end{proof}

\subsection{TARO with Adversary-Aware Optimization}
\label{sec:appx_taro_aa}

TARO is compatible with existing optimization-based purification objectives. In this section, we describe TARO-AA, a hybrid variant that integrates the temporal Product-of-Experts prior with the adversarial perturbation likelihood introduced by AAOpt~\citep{wesego2025adversary}. Standard AAOpt optimizes a single-scale diffusion prior together with an auxiliary perturbation model trained on adversarial residuals. TARO-AA replaces the single-scale prior with a multi-scale temporal consensus while retaining the perturbation likelihood. This allows the optimization to exploit two complementary sources of information: the pretrained clean-image diffusion model supplies scale-consistent image structure, while the perturbation model supplies an adversary-aware likelihood over residual artifacts.

\begin{algorithm}[!h]
  \caption{Test-time Optimization with AAOpt + TARO}
  \label{alg:aaopt_taro_combined}
  \begin{algorithmic}
    \State {\bfseries Input:} Input image \(\mathbf{x}_{\mathrm{adv}}\); pretrained score denoiser \(D_\theta\); perturbation network \(\mathbf{d}_\phi\); guidance strength \(\gamma\); perturbation weight \(\lambda_{\mathrm{pert}}\); iterations \(M\); timestep range \([t_{\min},t_{\max}]\); timestep sampling schedule; noise-level schedule \(\{\sigma_t\}_{t=1}^{T}\); perturbation model noise-level schedule \(\{\alpha\}_{t=1}^{T}\);ordered scale multipliers \([\kappa_1,\dots,\kappa_K]\)
    \State \(\mathbf{x}_0 \leftarrow \mathbf{x}_{\mathrm{adv}}\)
    \For{\(i=0\) {\bfseries to} \(M-1\)}
        \State Sample base timestep \(t\) from the predefined timestep schedule
        \For{\(j=1\) {\bfseries to} \(K\)}
            \State 
            \(
            t_j \leftarrow \mathrm{clip}\!\left(\sqrt{t\,\kappa_j},t_{\min},t_{\max}\right)
            \) \Comment{Set effective timestep}
            \State Sample \(\boldsymbol{\epsilon}_j \sim \mathcal N(\mathbf 0,\mathbf I)\)
            \State 
            \(
            \mathbf{n}_j \leftarrow \sigma_{t_j}\boldsymbol{\epsilon}_j
            \)
            \State 
            \(
            \mathbf{x}_{t,j} \leftarrow \mathbf{x}_i+\mathbf{n}_j
            \)
            \State 
            \(
            \mathbf{u}_j \leftarrow D_\theta\!\left(\mathbf{x}_{t,j},\sigma_{t_j}\right)
            \) \Comment{Compute temporal expert}
        \EndFor
        
        \State 
        \(
        \mathbf{x}_{t,\mathrm{TARO}} \leftarrow \mathbf{x}_{t,K}+\sum_{j=1}^{K-1}\gamma^{K-j}(\mathbf{x}_{t,j}-\mathbf{x}_{t,j+1})
        \) \Comment{Form aggregated $\mathbf{x}_{t}$}
        \State 
        \(
        \boldsymbol{\epsilon}_{\mathrm{TARO}} \leftarrow \mathbf{n}_{K}+\sum_{j=1}^{K-1}\gamma^{K-j}(\mathbf{n}_{j}-\mathbf{n}_{j+1})
        \) \Comment{Form aggregated noise}
        \State 
        \(
        \mathbf{u}_{\mathrm{TARO}} \leftarrow \mathbf{u}_{K}+\sum_{j=1}^{K-1}\gamma^{K-j}(\mathbf{u}_{j}-\mathbf{u}_{j+1})
        \) \Comment{Form TARO prediction}
        \State 
        \(
        \hat{\boldsymbol{\epsilon}} \leftarrow (\mathbf{x}_{t,\mathrm{TARO}}-\mathbf{u}_{\mathrm{TARO}})/\sigma_t
        \) \Comment{Estimate prior noise}
        \State 
        \(
        \ell_{\mathrm{prior}} \leftarrow \bigl(\mathrm{sg}(\hat{\boldsymbol{\epsilon}})-\boldsymbol{\epsilon}_{\mathrm{TARO}}\bigr)^\top \mathbf{x}_i
        \) \Comment{Calculate prior loss}
        
        \State \(\boldsymbol{\delta} \leftarrow \mathbf{x}_i-\mathbf{x}_{\mathrm{adv}}\)
        \State Sample \(\boldsymbol{\epsilon}' \sim \mathcal N(\mathbf 0,\mathbf I)\)
        \State 
        \(
        \boldsymbol{\delta}_t \leftarrow \sqrt{\alpha_t}\boldsymbol{\delta}+\sqrt{1-\alpha_t}\boldsymbol{\epsilon}'
        \)
        \State 
        \(
        \ell_{\mathrm{pert}} \leftarrow \bigl(\mathrm{sg}(\mathbf{d}_\phi(\boldsymbol{\delta}_t,t))-\boldsymbol{\epsilon}'\bigr)^\top \mathbf{x}_i
        \) \Comment{Calculate perturbation loss}
        
        \State 
        \(
        \ell \leftarrow \ell_{\mathrm{prior}}+\lambda_{\mathrm{pert}}\ell_{\mathrm{pert}}
        \) \Comment{Combine losses}
        \State 
        \(
        \mathbf{x}_{i+1} \leftarrow \mathrm{OptimizerUpdate}(\mathbf{x}_i,\ell)
        \) \Comment{Update image}
    \EndFor
    \State {\bfseries Return:} Final image \(\mathbf{x}_M\)
  \end{algorithmic}
\end{algorithm}

Let $[k_1,\ldots,k_K]$ denote the stream factors used to construct $K$ effective diffusion timescales. At each optimization step, TARO-AA evaluates the denoiser at these $K$ scales, aggregates the noisy states, injected noises, and denoised predictions using a recursive temporal weighting rule, and uses the resulting multi-scale estimate to define the diffusion prior loss. The perturbation loss is computed as in AAOpt from the residual $\boldsymbol{\delta}=\mathbf{x}-\mathbf{x}_{\mathrm{adv}}$. Algorithm~\ref{alg:aaopt_taro_combined} summarizes the full procedure.

\subsection{Proper Evaluation of Purification with Optimization}
\label{sec:appx_proper_eval}

Optimization-based purification requires a stricter evaluation protocol than standard reverse-sampling diffusion defenses because the purifier is itself an inner optimization procedure. Let $f:\mathbb{R}^d\to\mathbb{R}^K$ be the classifier and let $P:\mathbb{R}^d\to\mathbb{R}^d$ be an optimization-based purifier. Given a clean image $\mathbf{x}$ and perturbation $\boldsymbol{\delta}$ with $\|\boldsymbol{\delta}\|_p\leq\epsilon$, the adversarial input is $\mathbf{x}_{\mathrm{adv}}=\mathbf{x}+\boldsymbol{\delta}$. A correct white-box adversary should optimize the complete purifier--classifier composition:
\begin{equation}
    \max_{\|\boldsymbol{\delta}\|_p\leq\epsilon}
    \mathcal{L}_{\mathrm{cls}}
    \left(
        f(P(\mathbf{x}+\boldsymbol{\delta})),y
    \right).
\end{equation}
Thus, the relevant attack gradient is
\begin{equation}
    \mathbf{g}_{\mathrm{adv}}
    =
    \nabla_{\mathbf{x}_{\mathrm{adv}}}
    \mathcal{L}_{\mathrm{cls}}
    \left(
        f(P(\mathbf{x}_{\mathrm{adv}})),y
    \right),
\end{equation}
rather than the classifier gradient after purification.

For test-time optimization purifiers, $P(\mathbf{x}_{\mathrm{adv}})$ is not a single forward pass. The purifier initializes $\mathbf{z}_0=\mathbf{x}_{\mathrm{adv}}$ and performs $T$ inner updates using a guidance objective $\mathcal{L}_{\mathrm{guide}}$:
\begin{equation}
    \mathbf{z}_{t+1}
    =
    \mathbf{z}_t
    -
    \eta
    \nabla_{\mathbf{z}_t}
    \mathcal{L}_{\mathrm{guide}}(\mathbf{z}_t;\phi),
    \qquad
    t=0,\ldots,T-1,
\end{equation}
where $\phi$ denotes the parameters of the guidance model, such as a diffusion denoiser, score model, CLIP model, or adversarial perturbation model. The final purified image is $\mathbf{x}^{\star}=P(\mathbf{x}_{\mathrm{adv}})=\mathbf{z}_T$.

A correct attack must differentiate through the full dependence of $\mathbf{z}_T$ on $\mathbf{z}_0=\mathbf{x}_{\mathrm{adv}}$. Applying the chain rule gives
\begin{equation}
    \frac{\partial \mathcal{L}_{\mathrm{cls}}}{\partial \mathbf{x}_{\mathrm{adv}}}
    =
    \frac{\partial \mathcal{L}_{\mathrm{cls}}}{\partial f}
    \frac{\partial f}{\partial \mathbf{z}_T}
    \frac{\partial \mathbf{z}_T}{\partial \mathbf{z}_0}.
\end{equation}
The final factor is the trajectory Jacobian of the inner purification process. For the gradient-descent update above, each local Jacobian satisfies
\begin{equation}
    \frac{\partial \mathbf{z}_{t+1}}{\partial \mathbf{z}_t}
    =
    \mathbf{I}
    -
    \eta
    \nabla_{\mathbf{z}_t}^{2}
    \mathcal{L}_{\mathrm{guide}}(\mathbf{z}_t;\phi).
\end{equation}
Therefore, differentiating through an optimization-based purifier requires preserving second-order information from the inner guidance gradient. More generally, if the purifier uses an optimizer update
\begin{equation}
    \mathbf{z}_{t+1}
    =
    U_t(\mathbf{z}_t,\mathbf{g}_t),
    \qquad
    \mathbf{g}_t
    =
    \nabla_{\mathbf{z}_t}
    \mathcal{L}_{\mathrm{guide}}(\mathbf{z}_t;\phi),
\end{equation}
then the chain rule gives
\begin{equation}
    \frac{\partial \mathbf{z}_{t+1}}{\partial \mathbf{z}_t}
    =
    \frac{\partial U_t}{\partial \mathbf{z}_t}
    +
    \frac{\partial U_t}{\partial \mathbf{g}_t}
    \nabla_{\mathbf{z}_t}^{2}
    \mathcal{L}_{\mathrm{guide}}(\mathbf{z}_t;\phi).
\end{equation}
Thus, when the inner guidance gradient is treated as a constant, the attack omits the Hessian-dependent term and no longer corresponds to a true white-box attack on the optimization-based purifier.

This failure mode is easy to introduce in PyTorch. If the guidance gradient is computed inside the purification loop with the default setting,
\begin{equation}
    \texttt{torch.autograd.grad}
    \left(
        \mathcal{L}_{\mathrm{guide}},
        \mathbf{z}_t
    \right),
\end{equation}
then \texttt{create\_graph=False} by default, and the returned gradient is detached from the graph needed by the outer attack. The inner update still changes the purified image, but the outer adversary cannot differentiate through how this update changes as $\mathbf{x}_{\mathrm{adv}}$ changes. In effect, the attacker sees an incomplete surrogate of the purifier. A graph-consistent unrolled evaluation must instead preserve the higher-order path:
\begin{equation}
    \texttt{torch.autograd.grad}
    \left(
        \mathcal{L}_{\mathrm{guide}},
        \mathbf{z}_t,
        \texttt{create\_graph=True},
        \texttt{retain\_graph=True}
    \right).
\end{equation}
Here, \texttt{create\_graph=True} is the essential flag because it keeps the gradient computation differentiable; \texttt{retain\_graph=True} is needed when the same graph is reused for additional backward passes.

This graph-break issue is distinct from stochastic gradient masking. Diffusion purifiers often contain stochastic denoising steps, so expectation over transformation (EOT) is required to average gradients over noise realizations. However, EOT alone does not repair a detached inner optimization graph. A defense may use many EOT samples and still be mis-evaluated if the attacker cannot differentiate through the optimization trajectory that produced each purified sample. Therefore, optimization-based purification requires both stochasticity-aware evaluation and graph-consistent differentiation.

Our evaluation follows this principle. PGD-EOT attacks the full purifier--classifier composition, while full reverse-optimization AutoAttack-EOT in Appendix~\ref{sec:autoattack_eot} provides an additional check in which the attack differentiates through the complete optimization trajectory. The corrected CLIPure ~\cite{zhang2025clipure} case study below illustrates why this requirement is essential.

\subsection{Corrected Evaluation Case Study: CLIPure}
\label{sec:appx_clipure_eval}

We revisit CLIPure ~\cite{zhang2025clipure} as a case study in graph-consistent evaluation. CLIPure performs purification through test-time optimization using either a cosine-similarity objective in CLIP space or a diffusion-based objective. Since the purified image is produced by an inner optimization loop, a white-box attack must differentiate through that loop. If the inner guidance gradient is computed without preserving higher-order dependencies, the purification update becomes partially detached from the adversarial input, and the attack receives an incomplete gradient.

This produces accidental gradient masking. The defense appears robust not because the purified image is invariant to adversarial perturbations, but because the attacker is unable to observe how the purification trajectory changes as $\mathbf{x}_{\mathrm{adv}}$ changes. Restoring the computational graph through the inner loop exposes the true bilevel structure of the attack.

Table~\ref{tab:clipure_corrected_eval} compares the originally reported CLIPure robust accuracies with corrected graph-consistent evaluation on CIFAR-10 under AutoAttack with $\ell_\infty$ perturbations and $\epsilon=8/255$. Under corrected evaluation, CLIPure-Cos drops from 91.1\% to 0.0\%, and CLIPure-Diff drops from 88.0\% to 0.0\%.

\begin{table}[!h]
\centering
\caption{Reported versus corrected robust accuracy for CLIPure variants on CIFAR-10 under AutoAttack ($\ell_\infty$, $\epsilon=8/255$). Corrected evaluation preserves the computational graph through the inner optimization loop.}
\label{tab:clipure_corrected_eval}
\small
\setlength{\tabcolsep}{6pt}
\begin{tabular}{@{}l c c@{}}
\toprule
Method & Reported Robust Acc. & Corrected Robust Acc. \\
\midrule
CLIPure-Cos & 91.1\% & 0.0\% \\
CLIPure-Diff & 88.0\% & 0.0\% \\
\bottomrule
\end{tabular}
\end{table}

These results show that optimization-based purification can be severely mis-evaluated when higher-order gradient flow is not preserved. They do not imply that test-time optimization defenses are inherently ineffective; rather, they show that robustness must be established under graph-consistent attacks. In the next section, we therefore report AutoAttack-EOT with full reverse optimization for TARO and the strongest optimization baselines, providing a stricter adaptive evaluation beyond the one-step gradient approximation used in the main PGD-EOT experiments.

\begin{table}[!h]
\centering
\caption{Standard and robust accuracy on CIFAR-10 using WRN-70-16 ($\ell_\infty$, $\epsilon=8/255$) under PGD-EOT-20. Baselines include adversarial training (AT), adversarial purification, and test-time optimization purification.}
\label{tab:pgd_cifar10_linf_wrn70}
\small
\centering
\begin{tabular}{@{}l l c c@{}}
\toprule
Type & Method & Clean & Robust \\
\midrule
& Default & 95.19 & 0.0 \\
\midrule
\multirow{2}{*}{AT} 
& \citet{gowal2021improving} & 88.75 & 69.03 \\
& \citet{wang2023better} & 92.97 & 72.46 \\
\midrule
\multirow{5}{*}{AP} 
& \citet{yoon2021adp} & 86.76 & 41.02 \\
& \citet{Nie2022DiffPure} & 92.15 & 57.03 \\
& \citet{chen_robclsDiff} & 87.89 & 71.68 \\
& Score-Opt-O & 92.61 & 68.86 \\
& Score-Opt-N & 95.11 & 70.20 \\
& \citet{rob_diff_guang} & 90.42 & 66.41 \\
\midrule
& AAOpt & 91.66$\scriptstyle\pm1.2$ & 85.05$\scriptstyle\pm0.8$ \\
\midrule
& TARO-2 & 93.84$\scriptstyle\pm1.4$ & \textbf{88.56}$\scriptstyle\pm1.3$ \\
& TARO-3 & 93.55$\scriptstyle\pm1.1$ & 88.37$\scriptstyle\pm1.1$ \\
& TARO-2-AA & 92.80$\scriptstyle\pm0.9$ & 86.91$\scriptstyle\pm1.1$ \\
& TARO-3-AA & 93.84$\scriptstyle\pm1.0$ & 87.59$\scriptstyle\pm1.1$ \\
\bottomrule
\end{tabular}
\end{table}

\subsection{Additional Results on CIFAR-10 Attack Using WRN-70-16 Classifier}
\label{sec:appx_wrn_70_16}

In addition to the main WRN-28-10 evaluation, we report results using a second classifier architecture in this section. Table~\ref{tab:pgd_cifar10_linf_wrn70} evaluates CIFAR-10 under PGD-EOT-20 with a WRN-70-16 classifier, providing an additional check that TARO's gains are not tied to a single backbone. TARO remains strong under this larger classifier: TARO-2 achieves the best robust accuracy at $88.56\%$ while maintaining high clean accuracy, and TARO-3 obtains a comparable robust accuracy of $88.37\%$. Both variants substantially improve over AAOpt and prior purification baselines, and they also exceed the robust accuracy of adversarially trained WRN-70-16 classifiers listed in the table. The hybrid TARO-AA variants remain competitive, but the zero-shot TARO variants perform best in this $\ell_\infty$ setting. Overall, the WRN-70-16 results confirm that temporal multi-scale purification transfers across classifier backbones and remains competitive with both adversarial training and prior purification methods.

\subsection{Imagenette Results}
\label{sec:imagenette_results}

In addition to the CIFAR datasets, we evaluate TARO on Imagenette, a higher-resolution subset of ImageNet. Table~\ref{tab:pgd_eot_imagenette_taro} shows that TARO remains effective on a more complex image distribution. TARO-2 achieves the highest robust accuracy, while maintaining clean accuracy close to the undefended classifier and improving substantially over Score-Opt-O. TARO-2-AA also improves over AAOpt in clean accuracy and remains competitive in robustness. These results indicate that temporal rectification is not limited to low-resolution CIFAR images and can scale to more challenging image distributions without severe clean-accuracy degradation.

\begin{table}[!h]
  \centering
  \small
  \parbox[t]{0.99\textwidth}{
  \centering

  \captionof{table}{PGD-EOT attack on Imagenette ($\ell_\infty$, $\epsilon=8/255$), WRN-28-10.}
  \label{tab:pgd_eot_imagenette_taro}
  \setlength{\tabcolsep}{5pt}
  \begin{tabular}{@{}l c c@{}}
  \toprule
  Method & Clean & Robust \\
  \midrule
  Default & 82.9 & 0.0 \\
  \midrule
  Score-Opt-O & 76.4$\scriptstyle\pm1.2$ & 69.5$\scriptstyle\pm1.7$ \\
  AAOpt & 80.4$\scriptstyle\pm1.4$ & 66.4$\scriptstyle\pm1.9$ \\
  \midrule
  TARO-2 & 81.6$\scriptstyle\pm1.4$ & \textbf{72.3}$\scriptstyle\pm2.2$ \\
  TARO-2-AA & 81.2$\scriptstyle\pm1.6$ & 67.3$\scriptstyle\pm1.5$ \\
  \bottomrule
  \end{tabular}
  
  }
\end{table}

\subsection{Time-Scale Scheduling Ablation}
\label{sec:time_scale_ablation}

We further evaluate the generality of TARO across different time-scale scheduling choices. The core mechanism of TARO is multi-scale temporal aggregation: the purifier combines complementary denoising experts evaluated at different noise regimes. This mechanism should not depend on a single hand-designed schedule for selecting the noisier expert. In the main experiments, we use the square-root schedule $\tau=\sqrt{2t}$, which scales the variance-like diffusion parameter by a factor of $2$. To test whether TARO remains effective under an alternative parameterization, we also evaluate a linear schedule that directly scales the diffusion time as $\tau=2t$ for the noisier expert. Table~\ref{tab:time_scale_ablation_cifar10_linf} shows that TARO remains robust under both schedules, with the linear schedule achieving comparable clean accuracy and slightly higher robust accuracy. These results indicate that TARO's gains arise from temporal multi-scale rectification rather than a brittle dependence on a particular time-scaling rule.

\begin{table}[!h]
\centering
\caption{Time-scale scheduling ablation on CIFAR-10 under PGD-EOT-20 ($\ell_\infty$, $\epsilon=8/255$), WRN-28-10. The square-root schedule uses $\tau=\sqrt{2t}$ for the noisy expert, while the linear schedule uses $\tau=2t$.}
\label{tab:time_scale_ablation_cifar10_linf}
\small
\setlength{\tabcolsep}{6pt}
\begin{tabular}{@{}l c c@{}}
\toprule
Time-Scale Schedule & Clean Accuracy & Robust Accuracy \\
\midrule
Square-root scaling: $\tau=\sqrt{2t}$ & 93.41$\scriptstyle\pm1.1$ & 88.27$\scriptstyle\pm1.6$ \\
Linear scaling: $\tau=2t$ & 93.87$\scriptstyle\pm1.0$ & {89.05}$\scriptstyle\pm1.4$ \\
\bottomrule
\end{tabular}
\end{table}

\subsection{Robustness Against Common Corruptions}
\label{sec:corr}

We evaluate robustness to natural distribution shifts using CIFAR-10-C. Unlike adversarial attacks, these corruptions are not optimized against the classifier or purifier, but they test whether the multi-scale prior also improves robustness to common image degradations. Table~\ref{tab:robustness_cifar10c} reports accuracy on five corruption types using WRN-70-16. TARO-2 achieves the highest mean accuracy, outperforming Score-Opt-O and AAOpt. The improvement is especially strong on Gaussian noise, where TARO-2 reaches 92.0\% accuracy. The pure TARO-2 variant also slightly outperforms TARO-2-AA, suggesting that the unconstrained multi-scale consensus is particularly effective for non-adversarial corruptions where an adversarial residual likelihood is unnecessary.

\begin{table}[!h]
\centering
\caption{Robustness against common corruptions on CIFAR-10-C (WRN-70-16). The \textbf{Mean} column averages across the five corruption types.}
\label{tab:robustness_cifar10c}
\small
\setlength{\tabcolsep}{4pt}
\begin{tabular}{@{}lcccccc@{}}
\toprule
Method & Gaussian & Elastic & JPEG & Snow & Brightness & \textbf{Mean} \\
\midrule
Default & 45.6 & 82.9 & 75.8 & 83.5 & 93.8 & 76.3 \\
Score-Opt-O & 79.8 & 78.5 & 86.1 & 82.7 & 89.6 & 83.3 \\
AAOpt & 85.3 & 84.8 & 89.5 & 85.3 & 90.5 & 87.1 \\
\midrule
TARO-2 & \textbf{92.0} & 84.3 & \textbf{90.5} & \textbf{87.3} & \textbf{91.3} & \textbf{89.1} \\
TARO-2-AA & 91.3 & 81.1 & 88.6 & 86.5 & 89.4 & 87.4 \\
\bottomrule
\end{tabular}
\end{table}

\subsection{Computational Resources and Runtime}
\label{sec:compute}

Our experiments were conducted on NVIDIA A100 GPUs. Table~\ref{tab:optimization_speed} reports the wall-clock time required to complete a 5-step purification trajectory on CIFAR-10 with batch size 128. TARO introduces a modest increase in runtime relative to single-scale baselines because it evaluates multiple denoising experts per optimization step. However, this additional cost remains practical and is substantially smaller than the cost of increasing defense EOT or repeatedly extending a single-scale optimization trajectory.

\begin{table}[!h]
\centering
\caption{Runtime of purification optimization on CIFAR-10 using an NVIDIA A100 GPU.}
\label{tab:optimization_speed}
\small
\setlength{\tabcolsep}{6pt}
\begin{tabular}{@{}l c c@{}}
\toprule
Model & Batch Size & Time \\
\midrule
Score-Opt-O & 128 & 1.16 sec. \\
Score-Opt-N & 128 & 1.26 sec. \\
AAOpt & 128 & 1.05 sec. \\
TARO-2 & 128 & 1.38 sec. \\
TARO-2-AA & 128 & 1.45 sec. \\
\bottomrule
\end{tabular}
\end{table}

\subsection{Experiment Details}
\label{sec:exp_details}

We use 512 samples for each trial unless otherwise stated and report mean and standard deviation over trials. The DiffBreak LF evaluation is run on 64 samples averaged over 3 trials due to the cost of sequential EOT and 100 attack steps. For the main PGD-EOT experiments, we use 20 EOT samples for both attack and defense across TARO and directly reimplemented baselines. This differs from the original AAOpt evaluation, which uses 20 EOT samples for attack but a larger number of defense EOT samples.

For TARO-$K$, we use stream factors $[t,2]$ for $K=2$ and $[t,1,2]$ for $K=3$. Tables~\ref{tab:hyperparameter_taro2}, \ref{tab:hyperparameter_taro3}, and~\ref{tab:hyperparameter_taro2_aa} summarize the main hyperparameters used in our experiments. Additional implementation details are included in the supplementary code.

\begin{table}[!h]
\centering
\caption{Hyperparameters used for TARO-2.}
\label{tab:hyperparameter_taro2}
\small
\setlength{\tabcolsep}{4pt}
\begin{tabular}{@{}l c c c c c@{}}
\toprule
Attack Type & Budget & LR & Iter. & $\gamma$ & Timesteps \\
\midrule
\multicolumn{6}{c}{\textbf{CIFAR-10}} \\
\midrule
PGD+EOT & $\ell_\infty$ $(\epsilon=8/255)$ & 0.02 & 20 & 1.4 & $[0.25,\ldots,0]$ \\
PGD+EOT & $\ell_2$ $(\epsilon=0.5)$ & 0.02 & 20 & 1.4 & $[0.25,\ldots,0]$ \\
BPDA+EOT & $\ell_\infty$ $(\epsilon=8/255)$ & 0.02 & 20 & 1.4 & $[0.5,\ldots,0]$ \\
LF+EOT & -- & 0.02 & 20 & 1.4 & $[0.25,\ldots,0]$ \\
\midrule
\multicolumn{6}{c}{\textbf{CIFAR-100}} \\
\midrule
BPDA+EOT & $\ell_\infty$ $(\epsilon=8/255)$ & 0.1 & 3 & 1.4 & $[0.25,\ldots,0]$ \\
\midrule
\multicolumn{6}{c}{\textbf{CIFAR-10-C}} \\
\midrule
Corruptions & -- & 0.02 & 20 & 1.4 & $[0.25,\ldots,0]$ \\
\bottomrule
\end{tabular}
\end{table}

\begin{table}[!h]
\centering
\caption{Hyperparameters used for TARO-3.}
\label{tab:hyperparameter_taro3}
\small
\setlength{\tabcolsep}{4pt}
\begin{tabular}{@{}l c c c c c@{}}
\toprule
Attack Type & Budget & LR & Iter. & $\gamma$ & Timesteps \\
\midrule
\multicolumn{6}{c}{\textbf{CIFAR-10}} \\
\midrule
PGD+EOT & $\ell_\infty$ $(\epsilon=8/255)$ & 0.02 & 20 & 1.1 & $[0.25,\ldots,0]$ \\
PGD+EOT & $\ell_2$ $(\epsilon=0.5)$ & 0.02 & 20 & 1.1 & $[0.25,\ldots,0]$ \\
BPDA+EOT & $\ell_\infty$ $(\epsilon=8/255)$ & 0.02 & 20 & 1.1 & $[0.5,\ldots,0]$ \\
\midrule
\multicolumn{6}{c}{\textbf{CIFAR-100}} \\
\midrule
BPDA+EOT & $\ell_\infty$ $(\epsilon=8/255)$ & 0.1 & 3 & 1.1 & $[0.25,\ldots,0]$ \\
\bottomrule
\end{tabular}
\end{table}

\begin{table}[!h]
\centering
\caption{Hyperparameters used for TARO-2-AA and TARO-3-AA.}
\label{tab:hyperparameter_taro2_aa}
\small
\setlength{\tabcolsep}{3pt}
\begin{tabular}{@{}l c c c c c c@{}}
\toprule
Attack Type & Budget & LR & Iter. & $\gamma$ & $\lambda_{\mathrm{pert}}$ & Timesteps \\
\midrule
\multicolumn{7}{c}{\textbf{CIFAR-10}} \\
\midrule
PGD+EOT & $\ell_\infty$ $(\epsilon=8/255)$ & 0.1 & 20 & 1.1 & 0.25 & Uniform$(0.15,0.35)$ \\
PGD+EOT & $\ell_2$ $(\epsilon=0.5)$ & 0.1 & 20 & 1.1 & 0.25 & Uniform$(0.15,0.35)$ \\
BPDA+EOT & $\ell_\infty$ $(\epsilon=8/255)$ & 0.02 & 20 & 1.1 & 0.25 & $[0.25,\ldots,0]$ \\
LF+EOT & -- & 0.02 & 20 & 1.1 & 0.25 & $[0.25,\ldots,0]$ \\
\midrule
\multicolumn{7}{c}{\textbf{CIFAR-100}} \\
\midrule
BPDA+EOT & $\ell_\infty$ $(\epsilon=8/255)$ & 0.02 & 20 & 1.1 & 0.4 & $[0.15,\ldots,0]$ \\
\midrule
\multicolumn{7}{c}{\textbf{CIFAR-10-C}} \\
\midrule
Corruptions & -- & 0.02 & 20 & 1.1 & 0.1 & $[0.25,\ldots,0]$ \\
\bottomrule
\end{tabular}
\end{table}

\subsection{Limitations and Discussion}
\label{sec:appx_limitations}

TARO improves diffusion-based test-time purification by aggregating denoising
estimates across multiple temporal regimes, but it has several limitations.
First, TARO is an empirical defense and does not provide certified robustness. The temporal PoE and affine aggregation interpretations provide intuition for why multi-scale denoising can improve purification, but the local Gaussian view is only a conditional explanation: it relies on assumptions such as approximate local Gaussianity, positive definiteness of the signed precision, and suitable ordering between fine and coarse experts. These arguments should therefore be viewed as explanations of the optimization behavior rather than formal
robustness guarantees.

Second, TARO increases inference cost relative to single-scale purification
methods because each optimization step evaluates multiple denoising experts.
Our runtime ablations suggest that this additional cost is moderate. Nevertheless, TARO remains more expensive than a standard classifier forward pass or a single denoising trajectory. This cost becomes more significant under stronger adaptive evaluations such as full-reverse AutoAttack-EOT and DiffBreak, where repeated stochastic evaluations are
required.

Third, TARO's performance depends on the quality of the pretrained diffusion
prior, the downstream classifier, and the choice of optimization
hyperparameters. In particular, the correction strength \(\gamma\) and the use of consistency regularization can depend on the threat model: our experiments show different behavior under \(\ell_\infty\) and \(\ell_2\) attacks. This suggests that automatic selection of temporal weights and stabilization strengths is an important direction for future work.

Finally, TARO is designed and evaluated as a test-time image purification
method for classification. Extending the temporal multi-expert prior to other
modalities, detection, segmentation, inverse problems, or certified defenses
remains an important direction for future work.

\paragraph{Societal impact.}
TARO contributes to the development of more reliable machine learning systems by improving robustness against adversarial manipulation at inference time, which is important for safety-critical applications where models may be exposed to malicious inputs. The proposed proper and reliable evaluation protocol also has a positive societal impact by discouraging misleading robustness claims and promoting more reliable security evaluation. At the same time, stronger purification defenses may increase the computational cost and energy use of deployed systems, especially when repeated stochastic evaluations are required. We view TARO as part of a broader effort toward transparent, reproducible, and carefully evaluated defense rather than as a complete solution to adversarial risk.


\end{document}